\documentclass{article}

\usepackage{bbm}
\usepackage{mathtools}
\usepackage[numbers]{natbib}
\usepackage{verbatim}
\usepackage{comment}
\usepackage[linesnumbered, ruled,vlined]{algorithm2e}
\usepackage{algorithm,algorithmic}

\usepackage{microtype}
\usepackage{graphicx}
\usepackage{subfigure}
\usepackage{booktabs} 

\usepackage{hyperref}

\usepackage[preprint]{nips_2018}

\usepackage{tikz}
\usetikzlibrary{intersections,shapes.arrows}

\usepackage{amsthm,amsmath,amssymb, amsfonts}
\newtheorem{theorem}{Theorem}[section]
\newtheorem{corollary}{Corollary}[section]
\newtheorem{lemma}[theorem]{Lemma}

\usepackage{comment}
\usepackage{mathrsfs}
\usepackage{enumitem}
\usepackage{bbm}

\usepackage{pythonhighlight}

\title{Structured Monte Carlo Sampling for Nonisotropic Distributions via Determinantal Point Processes}

\author{
   Krzysztof Choromanski\textsuperscript{*} \\
  Google Brain Robotics \\
  kchoro@google.com \\
  \And
   Aldo Pacchiano\textsuperscript{*} \\
  UC Berkeley \\
 pacchiano@berkeley.edu \\
  \And
  Jack Parker-Holder\textsuperscript{*} \\
  Columbia University \\
 jh3764@columbia.edu \\
  \And
   Yunhao Tang\textsuperscript{*} \\
  Columbia University \\
  yt2541@columbia.edu \\
}

\begin{document}

\maketitle

\begin{abstract}
We propose a new class of structured methods for Monte Carlo (MC) sampling, called DPPMC, designed for high-dimensional nonisotropic distributions where samples are correlated to reduce the variance of the estimator via determinantal point processes. 
We successfully apply DPPMCs to problems involving nonisotropic distributions arising in guided evolution strategy (GES) methods for RL, CMA-ES techniques and trust region algorithms for blackbox optimization, improving state-of-the-art in all these settings. 
In particular, we show that DPPMCs drastically improve exploration profiles of the existing evolution strategy algorithms.
We further confirm our results, analyzing random feature map estimators for Gaussian mixture kernels. We provide theoretical justification of our empirical results, showing a connection between DPPMCs and structured orthogonal MC methods for isotropic distributions.
\end{abstract}

\section{Introduction}

Structured Monte Carlo (MC) sampling has recently received significant attention
\cite{choromanski_ort, choromanski_geo, choromanski_stockholm, unrea, montreal, kama, wasserstein} as a universal 
tool to improve MC methods for applications ranging from dimensionality 
reduction techniques and random feature map (RFM) kernel approximation \cite{unrea, kama} to evolution strategy methods for reinforcement learning (RL) \cite{montreal, wasserstein} and
estimating sliced Wasserstein distances between high-dimensional probabilistic distributions \cite{wasserstein}.
Structured MC methods rely on choosing samples from joint distributions where different samples are correlated in a particular way to reduce the variance of the estimator.
They are also related to the class of \textit{Quasi Monte Carlo} (QMC) methods that aim to improve concentration properties of MC estimators by using low discrepancy sequences of samples to reduce integration error \cite{quasi, kritzer}.

However, the key limitation of the above techniques is that they can only be applied to isotropic distributions, since they rely on samples' orthogonalization. 
For this class of methods the unbiasedness or asymptotic near-unbiasedness (for large enough dimensionality $d$) of the resulted orthogonal estimator follows directly
from the isotropicity of the corresponding multivariate distribution. 

We propose a new class of structured methods for MC sampling, called DPPMC, designed for high-dimensional non-isotropic distributions where samples are correlated to reduce the variance of the estimator via learned or non-adaptive determinantal point processes (DPPs) \cite{kulesza, gartrell}. DPPMCs are designed to work with highly non-isotropic distributions, yet they inherit
accuracy gains coming from structured estimators for the isotropic ones. As opposed to other sampling mechanisms using DPPs \cite{nystrom, nystrom2}, we propose a general hybrid DPP-MC architecture that can be applied in a wide range of scenarios from kernel estimation to RL.

We successfully applied DPPMCs to problems involving high-dimensional nonisotropic 
distributions naturally arising in guided evolution strategy (GES) methods for RL \cite{metz}, CMA-ES techniques and trust region methods for blackbox optimization, improving state-of-the-art in all of these settings. 
In particular, we show that DPPMCs drastically improve exploration profiles of the existing evolution strategy algorithms.
We further confirm our results analyzing RFM-estimators for Gaussian mixture kernels \cite{wilson, samo}, presenting detailed comparison with state-of-the-art density quantization methods. We use MC sampling as a preprocessing step from which a DPP downsamples to construct a final set of samples.
Furthermore, we provide theoretical justification of our empirical results, 
showing a connection between DPPMCs and structured orthogonal MC methods for isotropic distributions.

To motivate our approach, we mention the striking result from \cite{bardenet} showing that mixing quadratures with repulsive sampling provided by DPPs provably improves convergence rates of MC estimators. However, our algorithm is different - we do not rely on sampling from DPPs associated with multivariate orthogonal polynomials which requires cubic time. To the best of our knowledge, we are also the first to provide an extensive empirical evaluation showing that our approach is not only theoretically sound, but leads to efficient algorithms across a variety of settings. 



This paper is organized as follows: \textbf{(1)} In Section 2 we introduce Monte Carlo methods and Determinantal Point Processes, \textbf{(2)} In Section 3 we introduce our DPPMC algorithm, \textbf{(3)} In Section 4 we present theoretical guarantees for the class of DPPMC estimators, \textbf{(4)} In Section 5 we present all experimental results, in particular applications to a wide spectrum of reinforcement learning tasks.

\section{Towards DPPMCs: MC Methods and Determinantal Point Processes}

\subsection{Unstructured and Structured MC Sampling}
\label{sec:general}

Consider a function $F:\mathbb{R}^{d} \rightarrow \mathbb{R}^{m}$ defined as follows:
\begin{equation}
\label{base_eq}
F(\theta) = \mathbb{E}_{\mathbf{v} \sim \mathcal{D}}[h_{\theta}(\mathbf{v})], 
\end{equation}

where: $\mathcal{D} \in \mathcal{P}(\mathbb{R}^{d})$ is a $d$-dimensional (not necessarily isotropic) distribution and $h_{\theta}:\mathbb{R}^{d} \rightarrow \mathbb{R}^{m}$ is some function. Several important machine learning quantities can be expressed as in Equation \ref{base_eq}. For instance, many classes of kernel functions 
$K:\mathbb{R}^{d} \times \mathbb{R}^{d} \rightarrow \mathbb{R}$ admit
representation given by Equation \ref{base_eq}. 
The celebrated Bochner's theorem \cite{rahimi} states for every shift-invariant kernel $K:\mathbb{R}^{d} \times \mathbb{R}^{d} \rightarrow \mathbb{R}$:
\begin{equation}
\label{bochner_eq}
K(\mathbf{x},\mathbf{y}) = \int_{\mathbf{R}^{d}} p_{\mathcal{D}}(\omega)e^{i \omega^{T}(\mathbf{x}-\mathbf{y})}d\omega,
\end{equation}
for some distribution $\mathcal{D} \in \mathcal{P}(\mathbb{R}^{d})$ with density function $p_{\mathcal{D}}$
(sometimes called \textit{spectral density}) which is a Fourier Transform of $k:\mathbb{R}^{d} \rightarrow \mathbb{R}$ defined as $k(\tau)=K(\tau, 0)$. 
According to Equation \ref{bochner_eq}, values of the stationary kernel $K$ can be written as:
$
K(\mathbf{x},\mathbf{y}) = \mathbb{E}_{\mathbf{v} \sim \mathcal{D}} [\cos(\mathbf{v}^{\top}(\mathbf{x}-\mathbf{y}))],  
$
for some distribution $\mathcal{D} \in \mathcal{P}(\mathbb{R}^{d})$. If furthermore a stationary kernel $K$ is a radial basis function (RBF) kernel, i.e.
there exists $g:\mathbb{R} \rightarrow \mathbb{R}$ such that $K(\mathbf{x},\mathbf{y}) = g(\|\mathbf{x}-\mathbf{y}\|_{2})$, then the above 
distribution is isotropic. RBF kernels include in particular the classes of Gaussian, Mat\'{e}rn and Laplace kernels. Other prominent classes of kernels such as angular kernels or more general \textit{Pointwise Nonlinear Gaussian} kernels \cite{unrea} can be also expressed via Equation \ref{base_eq}.

Finally, in evolution strategies (ES), a blackbox optimization method frequently applied to learn policies for reinforcement 
learning and robotics \cite{ES, choromanski_stockholm, montreal, rbo},
gradients of Gaussian $\sigma$-smoothings of blackbox functions $f:\mathbb{R}^{d} \rightarrow \mathbb{R}$ (\textit{ES gradients}) are defined as:
\begin{equation}
\label{base_mc}
\nabla_{\sigma} f(\theta) = \mathbb{E}_{\mathbf{g} \sim \mathcal{N}(0, \mathbf{I}_{d})}[ \frac{1}{\sigma} f(\theta + \sigma \mathbf{g})\mathbf{g}].
\end{equation}

An unbiased baseline MC estimator of $F(\theta)$ from Equation \ref{base_eq} relies on independent sampling from distribution $\mathcal{D}$ and is of the form:
\begin{equation}
\widehat{F}^{\mathrm{iid}}_{m} = \frac{1}{m}\sum_{i=1}^{m} h_{\theta}(\mathbf{v}_{i}), 
\end{equation}
where $\mathbf{v}_{i} \overset{\mathrm{iid}}{\sim} \mathcal{D}$ and $m$ stands for the number of samples used. In the context of dot-product kernel approximation
that estimator leads to the so-called \textit{Johnson-Lindenstrauss Transforms} \cite{ailon2013, dasgupta2010sparse} 
and for nonlinear kernel approximation to the celebrated class of random feature map methods (see: \cite{rahimi}). 
In blackbox optimization domains it is a core part of many state-of-the-art ES methods \cite{ES, horia, rbo}.

In all the above applications distributions $\mathcal{D}$ from which samples were taken are isotropic. 
For such $\mathcal{D}$, we can further enforce different samples to be exactly orthogonal, while preserving their marginal distributions. This leads to the class of the so-called \textit{orthogonal estimators} $\widehat{F}^{\mathrm{ort}}_{m}$ \cite{choromanski_ort}, often characterized by lower variance than their unstructured counterparts \cite{choromanski_geo, unrea} followed by downstream gains (in ES optimization \cite{choromanski_stockholm}, Wassterstein GAN and autoencoder algorithms \cite{wasserstein} or even complicated hybrid predictive state recurrent neural network architectures as in \cite{choromanski_downey}).

\subsection{The Landscape of Nonisotropic Distributions}

Two fundamental limitations of the class of estimators $\widehat{F}^{\mathrm{ort}}_{m}$ is that they need the underlying
distributions to be isotropic for their (near)unbiasedness and they require the number of samples to satisfy $m \leq d$. 
Unfortunately, in practice the number of MC samples $m$ required even for a relatively modest task
of spherical Gaussian kernel approximation with precision $\epsilon$ with any constant probability is of the order $\Omega(\frac{d}{\epsilon^{2}}\log(\frac{d}{\epsilon}))$
(see: \cite{rahimi}). That problem can be addressed by stacking independent orthogonal blocks of samples.
However the former problem cannot be solved since the geometry of orthogonal structured transforms is intrinsically intertwined with the 
isotropicity of $\mathcal{D}$. 

Nonisotropic distributions arise in many important applications of machine learning. 
Several classes of non-RBF kernels are used as a more expressive tool to apply 
Gaussian processes (GPs) for learning hidden representation in data \cite{wilson}. 
The effectiveness of GPs depends on the quality of the interpolation mechanism applying given kernel function.
As noticed in \cite{remes}, RBF kernels lead to neighborhood-dominated interpolation that is unable of modelling different parts of the input space in several domains such as: geostatistics, bioinformations, 
signal processing.

A much more expressive family of non-monotonic (yet still stationary) kernels can be obtained by
modelling corresponding spectral density (leading straightforwardly to 
MC estimators) with the use of Gaussian mixture distributions $\mathcal{D}$ that are no longer isotropic. 

To be more specific, take the family of \textit{Gaussian mixture kernels} defined as:
\begin{equation}
\label{gkernel}
K(\mathbf{x},\mathbf{y}) = \sum_{q=1}^{Q} w^{q} \prod_{i=1}^{d} \exp(-2 \pi^{2} \tau_{i}^{2}v^{q}_{i})\cos(2\pi \tau_{i}\mu^{q}_{i}),  
\end{equation}

where: $\mathbf{x}, \mathbf{y} \in \mathbb{R}^{d}$, $\tau = \mathbf{x} - \mathbf{y}$, $Q$ is the number of Gaussian mixture components,
weights $w^{q}$ define their relative contributions, and finally
$\mu^{q}$ and $\mathrm{Cov}^{q} = \mathrm{diag}(v^{q}_{1},...,v^{q}_{d})$ stand for the mean and covariance matrix of the $q^{th}$ component.
The spectral distribution for that class of kernels $\mathcal{D} = 
\mathcal{N}(\{w^{1},\mu^{1},\mathrm{Cov}^{1}\},...,\{w^{Q},\mu^{Q},\mathrm{Cov}^{Q}\})$ is a 
mixture Gaussian distributions with relative weights $\{w^{1},...,w^{Q}\}$, means $\{\mu^{1},...,\mu^{Q}\}$ and covariance matrices 
$\{\mathrm{Cov}^{1},...,\mathrm{Cov}^{Q}\}$ of different mixture components.
Thus the values of these kernels can be expressed as:
$
K(\mathbf{x},\mathbf{y}) = \mathbb{E}_{\mathbf{v} \sim \mathcal{D}}\cos(\mathbf{v}^{\top}(\mathbf{x}-\mathbf{y}))
$
for the nonisotropic $\mathcal{D}$ defined above.

Since mixtures of Gaussians are dense in the set of distribution functions (in a weak topology sense), by applying Bochner's theorem, we can conclude that Gaussian mixture kernels
are dense in the space of all stationary kernels. The generalizations of Gaussian mixture kernels were also proved to be dense in the space of all non-stationary kernels \cite{samo}. 

Nonisotropic distributions also play a very important role in blackbox optimization, for instance in the CMA-ES algorithm \cite{akimoto, hansen} to create the populations
of samples of parameters to be evaluated in each epoch of the algorithm. Finally, learned nonisotropic distributions are applied on a regular basis in guided ES algorithms for 
policy optimization
\cite{metz} that estimate gradients of Gaussian smoothings $\nabla_{\sigma} f(\theta)$ of the RL function $f$ by sampling from nonisotropic distributions. 

\subsection{Determinantal Point Processes}
\label{sec:dpp}

Consider a finite set of datapoints $\mathcal{X} = \{\mathbf{x}^{1},...,\mathbf{x}^{N}\}$, where $\mathbf{x}^{i} \in \mathbb{R}^{d}$. A \textit{determinantal point process}
is a distribution $\mathcal{P}$ over the subsets of of $\mathcal{X}$ such that for some real, symmetric matrix $\mathbf{K}$ indexed by the elements of $\mathcal{X}$ the following holds for every
$A \subseteq \mathcal{X}$:
\begin{equation}\label{equation::marginal_kernel}
\mathbb{P}(A \subseteq \mathcal{S}) = \mathrm{det}(\mathbf{K}_{A}), 
\end{equation}

where $\mathcal{S}$ is sampled from $\mathcal{P}$ and $\mathbf{K}_{A}$ stands for the submatrix of $\mathbf{K}$ obtained by taking rows and columns indexed by the elements of $A$.
Note that $\mathbf{K}$ is positive semidefinite since all principal minors $\mathrm{det}(\mathbf{K}_{A})$ are nonnegative.
Determinantal point processes (DPPs) satisfy several so-called \textit{negative dependence property} conditions, such as:
$
\mathbb{P}[\mathbf{x}^{i} \in \mathcal{S} | \mathbf{x}^{j} \in \mathcal{S}] <  \mathbb{P}[\mathbf{x}^{i} \in \mathcal{S}]
$
for $i \neq j$, which can be directly derived from their algebraical definition. This makes them an interesting mechanism in applications where the goal is to subsample a 
diverse set of samples from a given set. 
To see it even more clearly, we can consider a restricted class of DPPs, the so-called $\textit{L-ensembles}$ \cite{borodin}, where the probability that a particular subset $S$ is chosen satisfies:
\begin{equation}\label{equation::l_ensemble}
\mathbb{P}[\mathcal{S}=S] = \frac{\mathrm{det}\mathbf{L}_{S}}{\mathrm{det}(\mathbf{L}+\mathbf{I}_{N})}
\end{equation}
for some matrix $\mathbf{L}$ that as before, has to be positive semidefinite.
If we interpret $\mathbf{L}$ as a kernel matrix $\mathbf{L}=[\langle \phi(\mathbf{x}^{i}),\phi(\mathbf{x}^{j}) \rangle]_{i,j=1,...,N}$, 
where $\phi$ is a corresponding feature map and $\langle \rangle$ stands for the dot-product form in the corresponding Hilbert space, then
we see that under the DPP sampling process the sets of near-orthogonal samples in the Hilbert space are favorable over nearly-collinear ones.
For instance, if $\phi : \mathbb{R}^{d} \rightarrow \mathbb{R}^{m}$ for some $m < \infty$ (as it is the case for example for random feature map representations from \cite{rahimi})
then probabilities $\mathbb{P}[\mathcal{S}=S]$ are proportional to squared volumes of the parallelepipeds defined by feature vectors $\phi(x^{s})$ for $s \in S$.
Thus samples that are similar according to a given kernel are less likely to appear together in the subsampled set than those that correspond to the orthogonal elements in the 
corresponding Hilbert space (see also Subsection \ref{sec:conn_ort}). 

The DPPs described above construct subsampled sets of different sizes, but if a fixed-size subset is needed a variant of the DPP called a k-DPP can be used 
(see: \cite{kdpp}).

\section{DPPMC Algorithm}\label{section::DPPMC_ALGORITHM}

We propose to estimate the expression from Equation \ref{base_eq} by the following procedure. We first choose the number of samples $m$ that we will average over (as in a standard baseline MC method). 
We then conduct oversampling by sampling independently at random $m \rho$ samples from $\mathcal{D}$ for some
fixed multiplier $\rho > 1$ (which is the hyperparameter of the algorithm) to obtain set $S_{\mathrm{MC}}$. Optionally, we renormalize datapoints of $S_{\mathrm{MC}}$ so that they are all of equal lengths.
We then downsample from the  $S_{\mathrm{MC}}$ using $m$-DPP and get an $m$-element set $\mathcal{S}_{\mathrm{DPP}}$. Finally, we estimate $F(\theta)$ as:
\begin{equation}\label{equation::downsampled_dpps}
\widehat{F}(\theta)^{\mathrm{DPPMC}} =  \frac{1}{m} \sum_{\mathbf{v} \in \mathcal{S}_{\mathrm{DPP}}} h_{\theta}(\mathbf{v}).
\end{equation}

In most practical applications it suffices to use a DPP determined by a fixed kernel function (see for instance: \cite{divnets}) and we show in Section \ref{blackbox} this approach is successful for RL tasks. However, for completeness we also present a learning framework. In order to learn the right kernel determining matrix $\mathbf{L}$ for the DPP (see: Subsection \ref{sec:dpp}), we model
this kernel as $K(\mathbf{x},\mathbf{y}) =\langle \phi(\mathbf{x}),\phi(\mathbf{y}) \rangle$, where function $\phi$ is the output of the feedforward
fully connected neural network. 

There is an extensive literature on learning DPPs via learned mappings $\phi$ produced by neural networks (see: \cite{gartrell}).
However, most approaches focus on a different setting, where
the goal is to learn the DPP from the subsets it produces (via negative maximal log-likelihood loss functions). Our neural network training is conducted as follows.

We approximate distribution $\mathcal{D}$ by the Gaussian mixture distribution $\mathcal{D}_{\mathrm{GM}}$. In most interesting practical applications the nonisotropic distributions under consideration are already Gaussian mixtures (thus no approximation is needed), but in principle the method can also be applied to other nonisotropic distributions.
Then we fix a training set of datapoints $\mathcal{X}_{\mathrm{train}} \subseteq \mathbb{R}^{d}$. In practice we use publicly available datasets (see: Subsection \ref{kernel_estimation}) with dimensionalities matching that of distribution $\mathcal{D}$. One can also consider synthetic datasets.
Next we train the neural network to minimize the empirical mean squared error ($\mathrm{MSE}$) of the DDPMC estimator of the Gaussian mixture kernel from Equation \ref{gkernel} corresponding to $\mathcal{D}_{\mathrm{GM}}$ on the pairs of points from the training set $\mathcal{X}_{\mathrm{train}}$ (this is just one of many loss functions that can be effectively used here). 

For given datapoints $\mathbf{x},\mathbf{y} \in \mathbb{R}^{d}$, the empirical $\mathrm{MSE}$ of the DPPMC approximator $\widehat{K}$ of the Gaussian mixture kernel $K$ is given as:
$
\label{empirical_mse}
\widehat{\mathrm{MSE}}(\widehat{K}(\mathbf{x},\mathbf{y})) = \frac{1}{t} \sum_{i=1}^{t} [(\frac{1}{m} \sum_{\mathbf{v} \in S_{\mathrm{DPP}}^{i}}h_{\tau}(\mathbf{v})-K(\mathbf{x},\mathbf{y}))^{2}],
$
where $\tau = \mathbf{x}-\mathbf{y}$, $h_{\theta}(\mathbf{v}) = \cos(\mathbf{v}^{\top}\theta)$ and sets $S^{i}_{\mathrm{DPP}}$ are constructed by $t$ independent runs of the above procedure,
where $t$ is a fixed hyperparameter determining accuracy of the estimation of $\mathrm{MSE}(\widehat{K}(\mathbf{x},\mathbf{y}))$.
The final loss function that we backpropagate through is the average empirical MSE over pairs of points from $\mathcal{X}_{\mathrm{train}}$.

The empirical mean squared error of kernels associated with nonisotropic 
distributions under consideration was chosen on purpose as an objective function minimized during training. For isotropic distributions the orthogonal structure (see: discussion about 
$\widehat{F}^{\mathrm{ort}}_{m}$ in Subsection \ref{sec:general}) that was first introduced as an effective tool for minimizing mean squared error of associated kernels (via random feature map 
mechanism) was later rediscovered as superior to baseline methods in other downstream tasks, as we discussed in Subsection \ref{sec:general}.






\section{Theoretical Results}

In this section we consider functions $F: \mathbb{R}^d \rightarrow \mathbb{R}^m$ from Equation \ref{base_eq}. All proofs of the presented results are given in the Appendix. 
We start by showing that DPPs can be used to provably reduce the MSE of downsampled estimators. 
Let $\{\mathbf{v}^1, \cdots, \mathbf{v}^N \} \subseteq \mathbb{R}^d$ be $N$ evaluation points of 
$F$\footnote{An important special case is when $\mathbf{v}^i \sim \mathcal{D}$ for all $i$ although it is 
not necessary for some of the results in this section to hold.}. Consider the case where each datapoint $\mathbf{v}^i$ 
is selected as part of the estimator with probability $p_i$. More formally, let $\{\epsilon_i\}_{i=1}^N$ be an ensemble of 
Bernoulli random variable with values in $\{0,1\}$ and marginal probabilities $\{p_i\}_{i=1}^N$.  Define the unbiased downsampled estimator as:
\begin{equation}
    \hat{F}(\theta)_U = \frac{1}{N}\sum_{i = 1}^N \frac{\epsilon_i}{p_i} h_\theta(\mathbf{v}^i).
\end{equation}
Notice that $\mathbb{E}_{\{\epsilon_i\}}\left[  \hat{F}(\theta)_U   \right] =\frac{1}{N} \sum_{i=1}^N h_\theta(\mathbf{v}^i)$. 
Let $\{ w_i \}$ be a set of importance weights with $w_i > 0$. 
We show that ensembles of Bernoulli random variables $\{ \epsilon_i \}$ sampled from a DPP can yield downsampling estimators with better variance than 
if these are produced i.i.d. with $\epsilon_i \sim \text{Ber}(p_i)$. Let $\mathbf{K}$ be a marginal kernel matrix defining a DPP with marginal probabilities $\mathbf{K}_{i,i} = p_i$ and such that the ensemble follows the DPP process. We consider the following subsampled ES estimator:
\begin{equation}
    \hat{F}(\theta)_U^{\mathrm{DPP}} = \frac{1}{N}\sum_{i=1}^N \frac{ \epsilon_i }{p_i } h_\theta(\mathbf{v}^i), 
 \end{equation}
where $\{\epsilon_i \} \sim \text{DPP}(\mathbf{K})$. Recall that here we have: $\mathbb{E}\left[ \epsilon_i  \right] = \mathbf{K}_{i,i}$ 
and $\mathbb{E}\left[ \epsilon_i \epsilon_j  \right] = \mathbf{K}_{i,i}\mathbf{K}_{j,j} - \mathbf{K}_{i,j}^2$ for $i \neq j$. 
We define $\hat{F}(\theta)_U^{\mathrm{iid}}$ in the analogous way, where this time samples $\{ \epsilon_i \}$ are i.i.d. Bernoulli with parameters $p_i$. 
In the theorem below we assume that $N \geq d+2$:
\begin{theorem}\label{theorem::unbiased_dpp_variance_reduction}
If $p_i  <1 $ for all $i$, there exists a Marginal Kernel $\mathbf{K} \in \mathbb{R}^{N\times N}$ such that:
\begin{align}
\begin{split}
    \mathbb{E}_{\{\epsilon_i\} \sim \mathrm{DPP}(\mathbf{K})}\left[  \hat{F}(\theta)_U^{\mathrm{DPP}}  \right] = \mathbb{E}_{\{ \epsilon_i \}  \sim \{Ber(p_i) \} }\left[  \hat{F}(\theta)_U^{\mathrm{iid}}  \right] 
    = \frac{1}{N} \sum_{i=1}^N h_\theta(\mathbf{v}^i)
\end{split}    
\end{align}
and furthermore:
$\mathrm{Var}( \hat{F}(\theta)^{\mathrm{DPP}}_U  ) <  \mathrm{
Var}( \hat{F}(\theta)^{\mathrm{iid}}_U)$. 
\end{theorem}

Thus DPP-based mechanism provides more accurate estimators.
As a consequence of the above theorem, 
we obtain guarantees for estimators of gradients of Gaussian smoothings. Let $f: \mathbb{R}^d \rightarrow \mathbb{R}$ and let 
$f_\sigma(\theta) = \mathbb{E}_{\mathbf{g} \sim \mathcal{N}(0, \mathbf{I}_d)}[ F(\theta + \sigma \mathbf{g})\mathbf{g}]$ be its Gaussian smoothing. 
Let $\nabla f_\sigma(\theta)$ 
denote the ES gradient of $f$, as defined in equation \ref{base_mc}, and call $\hat{\nabla}_U^{\mathrm{iid}} f_\sigma(\theta)$ and $\hat{\nabla}_U^{\mathrm{DPP}} f_\sigma(\theta)$ 
the corresponding unbiased downsampled iid and DPP versions of the estimator of $\nabla f_\sigma(\theta)$.



\begin{corollary}
\label{corollary::unbiased_ES_variance_reduction}
Let $\mathbf{g}^1, \cdots, \mathbf{g}^N \sim \mathcal{N}(0, \mathbf{I}_d)$ be $N \geq d+2$ iid normally distributed perturbations and let $\{p_i\}_{i=1}^N$ such that $p_i < 1$ for all $i$ be an ensemble of downsampling parameters. For any $\theta \in \mathbb{R}^d$ there is a marginal kernel $\mathbf{K} \in \mathbb{R}^{N\times N}$ such that:
$
    \mathbb{E}\left[ \hat{\nabla}_U^{\mathrm{DPP}} f_\sigma(\theta)   \right] = \mathbb{E}\left[ \hat{\nabla}^{\mathrm{iid}}_U f_\sigma(\theta)  \right] = \nabla f_\sigma(\theta), 
$
\begin{equation*}
    \underbrace{\mathbb{E}\left[ \hat{\nabla}_U^{\mathrm{DPP}} f_\sigma(\theta)   \right]}_{I} = \underbrace{\mathbb{E}\left[ \hat{\nabla}^{\mathrm{iid}}_U f_\sigma(\theta)  \right]}_{II} = \nabla f_\sigma(\theta), 
\end{equation*}
where: the first expectation is taken with respect to both $\{ \mathbf{v}^i \} \sim \mathcal{N}(0, \mathbf{I}_d)$ and $\{ \epsilon_i \} \sim \mathrm{DPP}(\mathbf{K})$
and the second expectation is taken with respect to both $\{ \mathbf{v}^i \} \sim \mathcal{N}(0, \mathbf{I}_d)$ and $\{ \epsilon_i\} \sim \{Ber(p_i) \}$. 
The variance satisfies:
$
\mathrm{Var}_(\hat{\nabla}^{\mathrm{DPP}}_{U} f_\sigma(\theta))  <\mathrm{
Var}( \hat{\nabla}^{\mathrm{iid}}_{U} f_\sigma(\theta)   ),
$
\begin{equation*}
    \underbrace{\mathrm{Var}_(\hat{\nabla}^{\mathrm{DPP}}_{U} f_\sigma(\theta)) }_{I} <\underbrace{\mathrm{
    Var}( \hat{\nabla}^{\mathrm{iid}}_{U} f_\sigma(\theta)   )}_{II},
\end{equation*}
where the variance on the LHS of the inequality is computed with respect to 
$\{\epsilon_i \} \sim \mathrm{DPP}(\mathbf{K})$ and the variance on the RHS is computed with respect to $\{\epsilon_i\} \sim \{Ber(p_i)\}$.
\end{corollary}
This implies that provided we select an appropriate DPP-Kernel matrix $\mathbf{K}$, 
DPPMC yields an unbiased estimator of the gradient of the Gaussian smoothing $\nabla f_\sigma(\theta)$ of smaller variance than 
iid estimator. The proof of this theorem can be turned into a procedure to produce such a Kernel $\mathbf{K}$. When the probabilities $p_i=p$ for all $i$, the importance weighted estimator is equivalent (with high probability) to the downsampled estimators we use in Section \ref{section::experiments} that already outperform other methods.

\subsection{Connections with Orthogonality}
\label{sec:conn_ort}

In this section we formalize the intuition that the most likely sets sampled under a Determinantal Point Process correspond to subsets of the dataset with orthogonal features
in the kernel space. 
In \cite{choromanski_stockholm} the authors 
study the benefits of coupling sensing directions used to build ES estimators 
by enforcing orthogonality between the sampling directions while preserving Gaussian marginals. 
It can be shown this strategy provably reduces the variance of the resulting gradient estimators. 
We shed light on this phenomenon through the perspective of DPPs. 
In what follows assume $\mathcal{X} = \{ \mathbf{x}^1, \cdots, \mathbf{x}^N\}$ with $\mathbf{x}^i \in \mathbb{R}^d$ and 
let $\phi : \mathbb{R}^d  \rightarrow \mathbb{R}^D$ be a possibly infinite feature 
map $\phi$ defining a kernel.  

\begin{theorem}\label{theorem::orthogonal_connections}
Let $\mathbf{L}=[\langle \phi(\mathbf{x}^i), \phi(\mathbf{x}^j) \rangle]_{i,j} \in \mathbb{R}^{N \times N} $ be an $\mathbf{L}-$ ensemble, where $\| \Phi(\mathbf{x}^i)\|_{2} = 1$ 
for all $i \in [N]$. Let $k \in \mathbb{N}$ with $k \leq N$ and assume there exist $k$ samples $\mathbf{x}^{i_1}, \cdots, \mathbf{x}^{i_k}$ in $\mathcal{X}$ 
satisfying $\langle \phi(\mathbf{x}^{i_j}) , \phi(\mathbf{x}^{i_l})\rangle = 0$ for all $j, l \in [k]$. If $\mathbb{P}_k$ denotes the $\mathrm{DPP}$ measure over subsets 
of size $k$ of $[N]$ defined by $\mathbf{L}$, the most likely outcomes from $\mathbb{P}_k$ are the size-$k$ pairwise orthogonal subsets of $\mathcal{X}$.
\end{theorem}

\section{Experiments}\label{section::experiments}

We aim to address here the following questions: 
\textbf{(1)} Do DPPMCs help to achieve better concentration results for MC estimation? 
\textbf{(2)} Do DPPMCs provide benefits for downstream tasks?
To address \textbf{(1)}, we consider estimating kernels using random features. 
To address \textbf{(2)}, we analyze applications of DPPMCs for high-dimensional blackbox optimization. 
We present extended ablation studies regarding parameter $\rho$ in the Appendix (see: Section \ref{ro}).

\paragraph{Complexity:} We emphasize the conceptual simplicity of our algorithm. Improving state-of-the-art in the RL setting, where we fix an RBF kernel defining the DPP (i.e. learning is not needed) requires adding few lines of code (we include a generic 11-line example of standard DPP python implementation in Section \ref{code}). Learning a DPP follows the standard supervised framework. Sampling from DPPs requires a priori the eigen-decomposition of  matrix $\mathbf{L}$, however we use fast sub-cubic (k)-DPP sampling mechanisms \cite{kang,li}. For blackbox optimization, time complexity of DPP sampling was negligible in comparison with that for function querying. Thus wall-clock time is accurately measured by the number of timesteps/function evaluations and we show that DPPMC enhancements need substantially fewer of them. For kernel approximation, time complexity of estimating kernel values is exactly the same for the DPPMC and baseline estimator (and reduces to that of matrix-vector multiplication). DPPMC requires DPP sampling, but in that setting it is a one-time cost.

\subsection{Kernel Estimation}
\label{kernel_estimation}

We compare the accuracy of the baseline MC estimator of values of Gaussian mixture kernels from Equation \ref{gkernel} using independent samples ($\mathrm{IID}$) with those applying Quasi Monte Carlo methods ($\mathrm{QMC}$) \cite{avron2016quasi}, estimators based on state-of-the-art quantization methods: $\mathrm{DPQ}$ \cite{luxburg}, $\mathrm{DSC}$ \cite{krause} and our DPPMC mechanism. We applied different QMC estimators and on each plot show the best one. We compare empirical mean squared errors of the above methods. The results are presented on $\mathrm{cpu}$ dataset. DPP mechanism was trained on $\mathrm{wine}$
dataset. Mapping $\phi$ was encoded by standard feedforward fully connected neural network architectures with two hidden layers of size $h=40$ each and with $\mathrm{tanh}$ nonlinearities. We analyzed Gaussian mixture kernels with different number of components $Q$. Fig. \ref{fig:mse} shows that
in all settings, DPPMC substantially outperforms
all other methods. We did not include orthogonal sampling method, since it did not work for the considered kernels.

\begin{figure}
\begin{minipage}{0.99\textwidth}
	\subfigure[$Q=2$]{\includegraphics[keepaspectratio, width=0.23\textwidth]{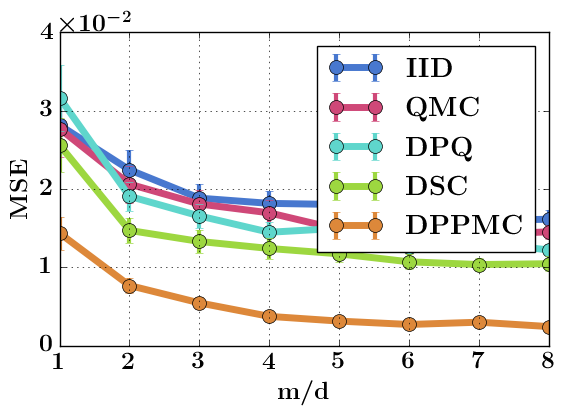}}
	\subfigure[$Q=3$]{\includegraphics[keepaspectratio, width=0.23\textwidth]{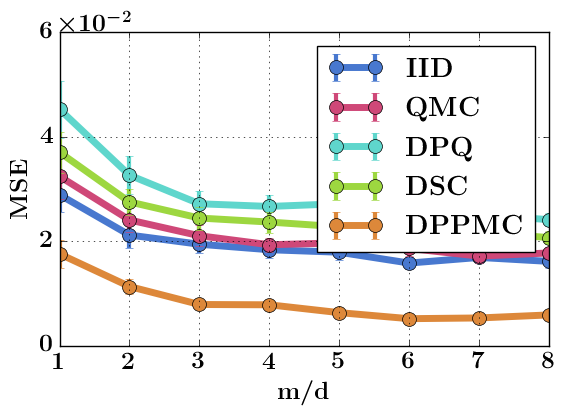}}  
	\subfigure[$Q=4$]{\includegraphics[keepaspectratio, width=0.23\textwidth]{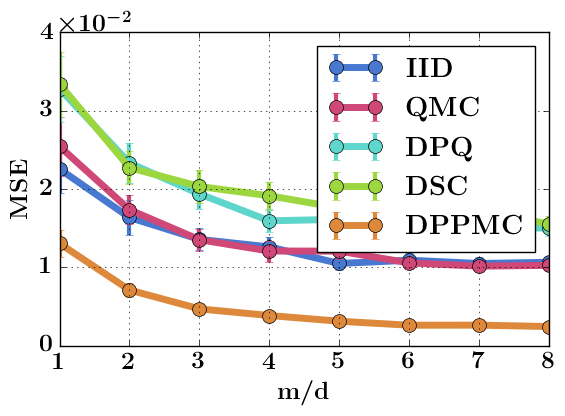}}
	\subfigure[$Q=5$]{\includegraphics[keepaspectratio, width=0.23\textwidth]{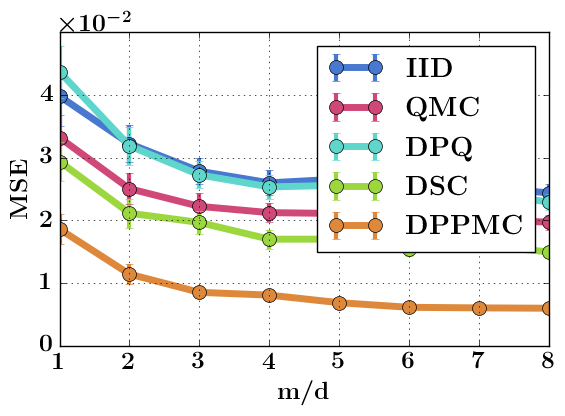}}  
\end{minipage}
	\caption{\small{Comparison of different estimators of Gaussian mixture kernels for different number of components: $Q$ on $\mathrm{cpu}$ dataset. On the horizontal axis: the ratio of the number of samples used and dimensionality of the
	datapoints. On the vertical axis: obtained empirical mean squared error.}}
	\label{fig:mse} 
\end{figure}

\subsection{Blackbox Optimization}
\label{blackbox}

ES blackbox optimization algorithms rely on sampling perturbation directions for function evaluations 
to optimize sets of parameters \citep{ES,choromanski_stockholm}. 
We propose to improve these baseline algorithms by augmenting their sampling subroutines with 
DPPMCs. We consider the following baseline methods:
\textbf{(1)} recently proposed guided ES methods, such as Guided Evolution Strategies \citep{metz}, \textbf{(2)} Trust-Region based ES methods resusing certain samples for better time complexity
\cite{rbo}, \textbf{(3)} Covariance Matrix Adaptation Evolution Strategy $\mathrm{CMA}$-$\mathrm{ES}$, a state-of-the-art blackbox optimization algorithm \cite{CMAES}. 

In each setting, the key difference between the baseline algorithm and our proposed method is that the former carries out uniform sampling 
from a given distribution $\mathcal{D}$, while our method diversifies the set of samples using DPPMC. Using a diverse set of samples leads to more efficient 
exploration in the parameter space and benefits downstream training, as we show later. We used a fixed Gaussian kernel with tuned variance to determine DPP.
We consider two sets of benchmark problems.


\paragraph{Reinforcement Learning:} In reinforcement learning (RL), at each time step $t$ an agent observes state $s_t \in \mathcal{S}$, takes action $a_t$, 
receives reward $r_t \in \mathbb{R}$ and transitions to the next state $s_{t+1} \in \mathcal{S}$. 
A policy is a mapping 
$\pi_\theta: \mathcal{S} \rightarrow \mathcal{A}$ from states to actions that will be conducted in that states and
is parameterized by vector $\theta$. The goal is to optimize that mapping to maximize expected cumulative reward $\mathbb{E}[\sum_{t=0}^T r_t]$ over given time horizon $T$. When framing RL as a blackbox optimization problem, the input $\theta$ to the blackbox function $f$ is usually a vectorized neural network 
and the output is a noisy estimate of the cumulative reward, 
obtained by executing policy $\pi_\theta$ in a particular environment. 
We consider environments: 
$\mathrm{Swimmer}$-$\mathrm{v2}$, $\mathrm{HalfCheetah}$-$\mathrm{v2}$, $\mathrm{Walker2d}$-$\mathrm{v2}$ and $\mathrm{Reacher}$
from the $\mathrm{OpenAI}$ $\mathrm{Gym}$ library and trained policies encoded by fully connected feedforward neural networks. 

\paragraph{Nevergrad Functions:} 
Blackbox functions from the recently open-sourced $\mathrm{Nevergrad}$ library \citep{nevergrad}, 
using the well-known open-source implementation of $\mathrm{CMA}$-$\mathrm{ES}$ (from $\mathrm{https://github.com/CMA}$-$\mathrm{ES/pycma}$). 
We tested functions: $\mathrm{Cigar}$, $\mathrm{Sphere}$, $\mathrm{Rosenbrock}$ and $\mathrm{Rastragin}$.

We are ready to describe the considered ES algorithms.
\begin{figure}[H]
\vspace{-3mm} 
\begin{minipage}{0.99\textwidth}
	\centering \subfigure{\includegraphics[keepaspectratio, width=0.26\textwidth]{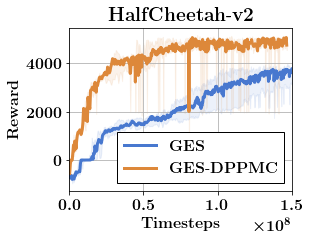}}
	\centering \subfigure{\includegraphics[keepaspectratio, width=0.23\textwidth]{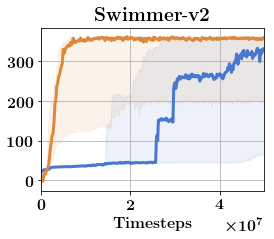}}
	\centering \subfigure{\includegraphics[keepaspectratio, width=0.26\textwidth]{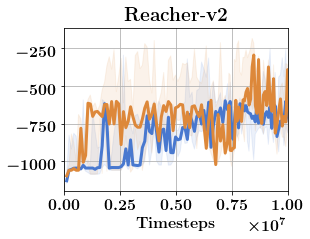}}
	\centering \subfigure{\includegraphics[keepaspectratio, width=0.23\textwidth]{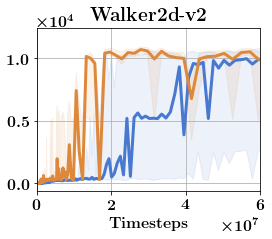}}  
\end{minipage}
	\caption{\small{Standard Guided ES versus their DPPMC enhancements on $\mathrm{OpenAI}$ $\mathrm{Gym}$ tasks.
	 Presented are median-curves from $k=10$ seeds and with inter-quartile ranges as shadowed regions.}}
	\label{fig:guided_exps} 
\vspace{-3mm} 	
\end{figure}

\paragraph{Guided ES:}
In each iteration, Guided ES methods sample $m$ perturbation vectors from the non-isotropic Gaussian distribution $\mathcal{D}$ with an adaptive covariance matrix computed from 
the empirical covariance matrix of gradients obtained via a biased oracle \citep{metz} or previous estimation, 
as it is the case in recently proposed approaches based on ES-active subspaces. Such an adaptive non-isotropic sensing 
often leads to more sample-efficient estimation of the gradient by exploring subspaces where the true gradients are most likely to be. 
In the DPPMC enhancement of those techniques, we first sample $l=\rho m$ vectors from $\mathcal{D}$ for $\rho=10$, and down-sample to get a subset of $m$ vectors via DPPs. 

In Fig.\ref{fig:guided_exps}, we compare baseline Guided ES with its enhanced DPPMC version. 
The vertical axis shows the expected cumulative reward during training and the horizontal 
axis - the number of time steps. Each plot shows the average performance with shaded area 
indicating inter-quartiles across $r=10$ random seeds. DPPMC leads 
to substantially better training curves. To achieve reward $\approx 2000$ in HalfCheetah-v2, baseline algorithm requires $\approx 10^8$ steps while DPPMC only $10^{7}$.



\paragraph{Trust Region ES:}

Trust Region ES methods, as those recently proposed in \cite{rbo}, rely on reusing $\delta m$ perturbations from previous epochs for some $0 < \delta < 1$
and applying regression techniques to estimate gradients of blackbox functions.
Those methods do not require perturbations to be independent.
DPPMCs can be applied in this setting by sampling $(1-\frac{\delta}{2})m$ new perturbations (instead of $(1-\delta)m$)
and then downsampling from the set of all $(1+\frac{\delta}{2})m$ perurbations ($(1-\frac{\delta}{2})m$ new and $\delta m$ reused) only $m$ of them. By doing it, we do not reuse all $\delta m$ samples, but obtain much more diverse set of perturbations that ultimately improves sampling complexity. We take $\delta=0.2$.

\begin{figure}[H]
\vspace{-3mm} 
\begin{minipage}{1\textwidth}
	\centering \subfigure{\includegraphics[keepaspectratio, width=0.268\textwidth]{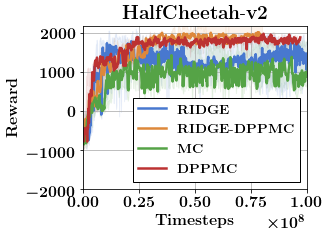}}
	\centering \subfigure{\includegraphics[keepaspectratio, width=0.225\textwidth]{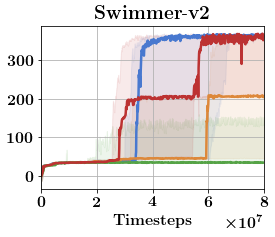}}  
	\centering \subfigure{\includegraphics[keepaspectratio, width=0.25\textwidth]{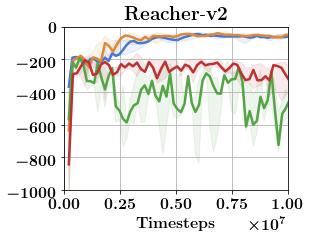}}
	\centering \subfigure{\includegraphics[keepaspectratio, width=0.237\textwidth]{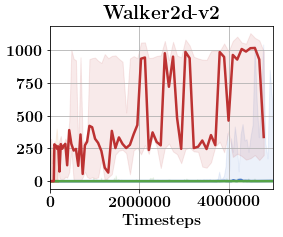}}  
\end{minipage}
	\caption{\small{$\mathrm{RBO}$ trust region method using MC/ridge gradients versus its DPPMC enhancements on $\mathrm{OpenAI}$ $\mathrm{Gym}$ tasks.
	 All curves are median-curves from $k=5$ seeds and with inter-quartile ranges as shadowed regions.}}
	\label{fig:tr_exps} 
\vspace{-3mm} 	
\end{figure}

As we can see in Fig.\ref{fig:tr_exps}, for most of the cases 
DPPMC-based  Trust Region ES method outperforms algorithm $\mathrm{RBO}$ from \cite{rbo}
that uses standard Trust Region ES mechanism and was already showed to outperform vanilla ES baselines. In particular, for $\mathrm{Walker2d}$-$\mathrm{v2}$ the only method that manages to learn in a given timeframe is based
on DPPMC sampling. 

\begin{figure}[H]
\vspace{-3mm} 
\begin{minipage}{0.99\textwidth}
	\centering \subfigure{\includegraphics[keepaspectratio, width=0.245\textwidth]{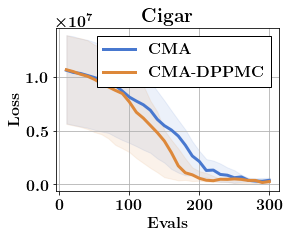}}
	\centering \subfigure{\includegraphics[keepaspectratio, width=0.225\textwidth]{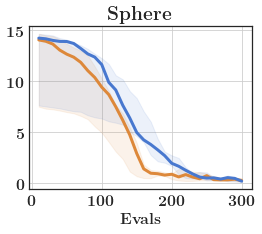}}  
	\centering \subfigure{\includegraphics[keepaspectratio, width=0.24\textwidth]{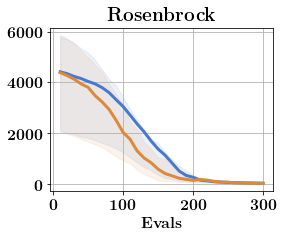}}
	\centering \subfigure{\includegraphics[keepaspectratio, width=0.235\textwidth]{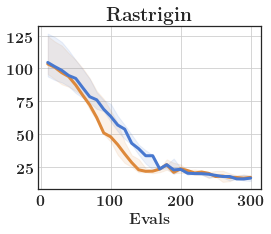}}  
\end{minipage}
	\caption{\small{$\mathrm{CMA}$-$\mathrm{ES}$ (baseline) versus its DPPMC version for $\mathrm{Nevergrad}$ functions. 
	Presented are median-curves from $k=5$ seeds and with inter-quartile ranges as shadowed regions.}}
	\label{fig:cmaes_exp} 
\vspace{-3mm} 	
\end{figure}

\paragraph{CMA-ES:}
In each iteration, $\mathrm{CMA}$-$\mathrm{ES}$ samples a set of $m$ perturbation vectors from a non-isotropic Gaussian distribution for function evaluations. 
Unlike for the above Guided ES methods, the covariance matrix is adapted by running weighted regression over sampled perturbations, where the weights are the 
function evaluations for different perturbations. Such an adaptive mechanism allows also for efficient exploration in the parameter space, and has performed robustly 
even for high-dimensional tasks \citep{CMAES,rllab}. To construct the candidate pool for $\mathrm{CMA}$-$\mathrm{ES}$, 
we first sample $l=\rho m$ non-isotropic Gaussian vectors for $\rho=10$, and then downsample $m$ elements via DPPs. 


We compare $\mathrm{CMA}$-$\mathrm{ES}$ baseline with its DPPMC enhancement 
in Fig. \ref{fig:cmaes_exp}. The horizontal axis shows the cumulative number of function evaluations we make as 
the optimization progresses, while the vertical axis shows the expected loss. Each plot shows the average performance with shaded area indicating 
inter-quartiles across $5$ random seeds. DPPMC achieves consistent gains across 
all presented $\mathrm{Nevergrad}$ benchmarks. 
We remark that since the  open source implementation of $\mathrm{pycma}$ is highly optimized, 
obtaining even marginal improvements across multiple benchmarks is not  trivial.


\vspace{-3mm} 
\section{Conclusion}

We presented new sampling mechanism DPPMC based on determinantal point processes to improve standard
MC methods for nonisotropic distributions. We furthermore showed the effectiveness of our approach on several downstream tasks
(guided ES search, CMA-ES and trust-region methods for blackbox optimization) and provided theoretical guarantees.

\bibliographystyle{abbrv}
\bibliography{main}

\onecolumn

\section{APPENDIX: Structured Monte Carlo Sampling for Nonisotropic Distributions via Determinantal Point Processes}

\subsection{Variance Reduction for Evolution Strategies using DPPs}

The goal of this section is to show that it is possible to use DPPs to reduce the variance of Evolution Strategies gradient estimators.

\subsubsection{One dimensional variance reduction using DPPs}

 We start by showing an auxiliary sequence of one dimensional lemmas. We consider the problem of computing an estimator of the sum $\bar{a}$ of $n$ real numbers $a_1, \cdots, a_n$.  In Lemma \ref{lemma::variance_lemma_1} we first show that using DPPs it is always possible to produce an unbiased estimator of the sum of a sequence of real numbers with less or equal variance than the i.i.d estimator that samples each element $a_i$ of the sequence i.i.d. with probability $p_i$. We then show in Lemma \ref{lemma::variance_reduction_lemma_2}  that it is possible to produce a DPP kernel $\mathbf{K}$ such that the corresponding sum estimator has strictly less variance than the i.i.d. one.

We follow the discussion regarding Determinantal Point Process from \cite{kulesza}. Recall that a Determinantal Point Pricess (DPP) $\mathcal{P}$ on a ground set $\mathcal{X}$ with $| \mathcal{X} | = N$ is a probability measure over power set $2^\mathcal{X}$. When $\mathcal{S}$ is a random subset drawn according to $\mathcal{P}$, we have, for every $A \subset \mathcal{X}$. 

\begin{equation*}
    \mathcal{P}\left( A \subset \mathcal{S}   \right) = \det\left(\mathbf{K}_A \right)
\end{equation*}

for some real symmetric $N\times N$ matrix $\mathbf{K}$ indexed by the elements of $\mathcal{X}$. Here $\mathbf{K}_{A} = [\mathbf{K}_{i,j}]_{i,j \in A}$ and adopt $\det(\mathbf{K}_\emptyset) = 1$. $\mathbf{K}$ is known as the marginal kernel. 

Notice that whenever $A = \{i\}$,  $\mathbb{P}(i \in \mathbf{S}) = \mathbf{K}_{i,i}$ and that $\mathbb{P}({i,j} \in \mathcal{S}) = \mathbb{P}(i \in \mathcal{S}) \mathbb{P}(j \in \mathcal{S}) - \mathbf{K}_{i,j}^2$.

We start by showing a basic variance reduction result regarding DPPs. Let $a_1, \cdots, a_n$ be set of real numbers. Let $\bar{a}$ be their sum. We are interested in analyzing the following two estimators of $\bar{a}$:

\begin{enumerate}
    \item $\hat{a}_{\mathrm{i.i.d}} = \sum_{i=1}^n \frac{a_i \epsilon_i}{p_i}$ where $\epsilon_i$ are sampled independent from each other with $\epsilon_i \sim \mathrm{Ber}(p_i)$. 
    \item $\hat{a}_{\mathrm{DPP}} = \sum_{i \in \mathcal{S}} \frac{a_i \epsilon_i }{p_i}$ where $\mathcal{S}$ is a subset of $[n]$ sampled from a DPP with kernel $\mathbf{K}$ satisfying $\mathbf{K}_{i,i} = p_i$ for all $i$. 
\end{enumerate}

Notice that $\mathbb{E}\left[  \hat{a}_{\mathrm{i.i.d}} \right] = \bar{a}$ and $\mathbb{E}\left[  \hat{a}_{\mathrm{DPP}} \right] = \bar{a}$ and therefore $\hat{a}_{\mathrm{i.i.d}}$ and $\hat{a}_{\mathrm{DPP}}$ are unbiased estimators of $\bar{a}$.

\begin{lemma}\label{lemma::variance_lemma_1}
If $a_i \geq 0$ for all $i$, the estimator $\hat{a}_{DPP}$ has smaller variance than $\hat{a}_{i.i.d}$ whenever $\mathbf{K}_{ii} = p_i$ for all $i$.
\end{lemma}

\begin{proof}


Since $\hat{a}_{i.i.d}$ and $\hat{a}_{\mathrm{DPP}}$ are unbiased, it is enough to compare the second moments of the said estimators.

\begin{align*}
    \mathbb{E}\left[  \hat{a}_{\mathrm{DPP}}^2 \right] &= \mathbb{E}\left[\sum_{i j} \frac{ a_ia_j \epsilon_i \epsilon_j  }{p_ip_j}      \right] \\
    &= \sum_{i,j} \frac{\mathbb{E}\left[ \epsilon_i\epsilon_j  \right] a_i a_j }{p_ip_j} \\
    &= \sum_{i,j} \frac{ (\mathbf{K}_{ii}\mathbf{K}_{jj} - \mathbf{K}_{ij}^2) a_ia_j   }{p_ip_j} \\
    &= \mathbb{E}\left[\hat{a}_{i.i.d.}^2   \right]  - \sum_{i \neq j}\frac{\mathbf{K}_{i,j}^2 a_ia_j}{p_ip_j} \\
    &\leq \mathbb{E}\left[ \hat{a}_{i.i.d.}^2   \right]
\end{align*}

The last inequality holds whenever $a_i \geq 0$ for all $i$.


\end{proof}

We can also show that under appropriate conditions there exists a kernel matrix $\mathbf{K}$ such that $\mathrm{Var}(\hat{a}_{\mathrm{i.i.d}}) > \mathrm{Var}(\hat{a}_{\mathrm{DPP}})$ such that the inequality is strict.

\begin{lemma}\label{lemma::variance_reduction_lemma_2}
If $n \geq 3$, $p_i > 0$ for all $i$and there exists $i$ such that $p_i < 1$, then there exists a matrix $\mathbf{K} \in \mathbb{R}^{n\times n}$ defining a DPP over  - not necessarily nonnegative- $a_1, \cdots, a_n \in \mathbb{R}$ satisfying $\mathbf{K}_{i,i} = p_i$ and such that  $\mathrm{Var}(\hat{a}_{\mathrm{i.i.d.}}) > \mathrm{Var}(\hat{a}_{\mathrm{DPP}})$.
\end{lemma}

\begin{proof}
Let $\mathbf{K}$ be a matrix defining a DPP with $\mathbf{K}_{i,i} = p_i$ for all $i$     Following the exact same proof as in Lemma \ref{lemma::variance_lemma_1}, we conclude that $\text{Var}(\hat{a}_{\mathrm{i,i.d}}) > \text{Var}(\hat{a}_{\mathrm{DPP}})$ iff:
    
    \begin{equation}\label{equation::variance_comparison_condition}
        \sum_{i\neq j}\frac{\mathbf{K}_{i,j}^2 a_i a_j}{p_i p_j} > 0
    \end{equation}

We show the existence of a kernel matrix $\mathbf{K}$ for which the inequality \ref{equation::variance_comparison_condition} holds and $\mathbf{K}_{i,i} = p_i$ for all $i$.

Indeed, let $\mathbf{K}  \in \mathbb{R}^{n\times n}$ be such that:

\begin{equation*}
    \mathbf{K}_{i,j} = \begin{cases}
        p_i & \text{if } i = j \\
        \epsilon & \text{if } \frac{a_i a_j}{p_ip_j}  \geq 0\\
        0 & \text{o.w.}
        \end{cases}
\end{equation*}

For some $\epsilon > 0$. Under this definition, notice that $\sum_{i,j} \frac{\mathbf{K}_{i,j}^2 a_i a_j} { p_i p_j} > 0$ and notice that since $0 \prec \text{diag}(p_i) \prec I$, there exists a choice of $\epsilon >0$ such that $0 \prec \mathbf{K} \prec \mathbb{I}_d $, thus defining a valid DPP kernel matrix $\mathbf{K}$.

\end{proof}

\subsubsection{Towards variance reduction for vector estimators using DPPs.}

In this section we extend the results of the previous section to the multi dimensional case of Monte Carlo gradient estimators.  We start with an auxiliary lemma that will be used in the variance reduction Theorems  of the following sections. The following Lemma characterizes the maximum number of vectors that can all be pairwise negatively correlated. This Lemma will be used later on to argue the existence of a DPP kernel $\mathbf{K}$ for which its subsampling estimator of the Evolution Strategies gradient estimator achieves less variance than the i.i.d. subsampling estimator.

\begin{lemma}\label{lemma::vectors_negative_correlations}
Let $\mathbf{v}^1, \cdots, \mathbf{v}^M \in \mathbb{R}^d$ vectors such that $\langle \mathbf{v}^i , \mathbf{v}^j \rangle < 0$ for all $i \neq j$. Then $M \leq d+1$.
\end{lemma}

\begin{proof}
We proceed with a proof by contradiction. Let's assume $M \geq d+2$. Let $\mathbf{v}^1, \cdots, \mathbf{v}^{d+1}$ be a subset of $d+1$ vectors of $\{\mathbf{v}^j \}_{j=1}^M$. There exist $a_1, \cdots, a_{d+1} \in \mathbb{R}$ such that:
\begin{equation*}
    \sum_{i=1}^{d+1} a_i \mathbf{v}^i = 0
\end{equation*}
If $a_i \geq 0$ for all $i$ then $\langle \mathbf{v}^{d+2}, \sum_{i} a_i \mathbf{v}^i \rangle = \sum_i a_i \langle \mathbf{v}^{d+2}, \mathbf{v}^i \rangle < 0$ which would result in a contradiction.
If $a_i$ are not all nonnegative, there exist disjoint subsets $I \subset [ d+2 ] $ and $K \subset [d+2]$ such that $I \cup J = [d+2]$, and $I \cap J = \emptyset$  and $I, J \neq \emptyset$ and with $a_i \geq 0$ for all $i \in I$ (with at least one $a_i > 0$) and $a_j \leq 0$ (with at least one $a_j < 0$) for all $j \in J$ such that:
\begin{equation*}
    \underbrace{\sum_{i \in I} a_i \mathbf{v}^i}_{I} =\underbrace{ \sum_{j\in J} -a_j \mathbf{v}^j }_{II}
\end{equation*}
Therefore by assumption $\langle I, II \rangle < 0$ which would cause a contradiction since $I = II$.

\end{proof}

Recall the gradient estimator corresponding to Evolution Strategies. If $f : \mathbb{R}^d \rightarrow \mathbb{R}$, the ES gradient estimator $\nabla f_\sigma(\theta)$ at $\theta$ equals:

\begin{equation*}
    \nabla f_\sigma(\theta) = \mathbb{E}_{\mathbf{v} \sim \mathcal{N}(0,\mathbb{I}_d)}\left[ \frac{1}{\sigma}f(\theta + \sigma \mathbf{v}) \mathbf{v}  \right]
\end{equation*}

We denote by $\hat{\nabla} f_\sigma(\theta)= \frac{1}{n\sigma} \sum_{i=1}^n f(\theta + \sigma \mathbf{v}^i) \mathbf{v}^i$ where $\mathbf{v}^i$ are all samples from a standard Gaussian $\mathcal{N}(0, \mathbb{I}_d)$.

\subsubsection{Subsampling strategies in ES}

In this section we consider subsampling strategies for Evolution strategies when we have a dictionary of $N$ sensing directions $\{\mathbf{v}^i\}_{i=1}^N$. Let $\{p_i\}_{i=1}^N$ be the ensemble of probabilities with which to sample (according to a Bernoulli trial with probability $p_i$) each sensing $i$. 

We recognize two cases:

\begin{enumerate}
    \item \textbf{Unbiased sampling} In this case we consider a subsampled-importance sampling weighted version of the empirical estimator $\hat{\nabla} f_\sigma(\theta) = \frac{1}{\sigma N}\sum_{i=1}^N \mathbf{v}^i f(\theta + \sigma \mathbf{v}^i)$ of the form $\hat{\nabla}_{ U} f_\sigma(\theta) = \frac{1}{N\sigma} \sum_{i=1}^N\frac{\epsilon_i}{p_i} \mathbf{v}^i f(\theta + \sigma \mathbf{v}^i)$.
    \item \textbf{Biased} In this case we consider a subsampled version of the empirical estimator $\hat{\nabla} f_\sigma(\theta)$ of the form $\hat{\nabla}_B f_\sigma(\theta) = \frac{1}{\sigma N} \sum_{i=1}^N \frac{\epsilon_i}{w_i} \mathbf{v}^i f(\theta + \sigma \mathbf{v}^i)$ where $\{w_i\}_{i=1}^N$ is a set of importance weights, not necessarily equal to $\{p_i\}$.
\end{enumerate}

The crucial observation behind these estimators is that the evaluation of $f$ need not be performed at the points that are not subsampled. This allows us to trade off computation with variance (or mean squared error). We would like to achieve the optimal tradeoff. 

\textbf{Unbiased subsampling}

The goal of this section is to show that for any i.i.d. subsampling strategy to build an unbiased estimator for the ES gradient, there exists a DPP kernel such that the DPP unbiased subsampling estimator achieves less variance than the i.i.d. one. 

The main result of this section, Theorem \ref{theorem::auxiliary_dpp_variance_1} concerns the estimation of functions of the form $F : \mathbb{R}^d \rightarrow \mathbb{R}^m$ as defined in Section \ref{base_eq}, and shows that for any fixed subsampling i.i.d. strategy (encoded by subsampling probabilities $\{p_i\}$), there exists a marginal kernel $\mathbf{K}$ whose corresponding estimator achieves the same mean but has (strictly) less variance. We prove Theorem \ref{theorem::unbiased_dpp_variance_reduction_evolution_strategies}   which specializes Theorem \ref{theorem::auxiliary_dpp_variance_1} to the case of ES gradients. A simple notational change would render the proof valid for Theorem \ref{theorem::auxiliary_dpp_variance_1}.  

The following corresponds to Theorem \ref{theorem::unbiased_dpp_variance_reduction} in the main text.

\begin{theorem}\label{theorem::auxiliary_dpp_variance_1}
If $N \geq d+2$ and $p_i  <1 $ for all $i$, there exists a Marginal Kernel $\mathbf{K} \in \mathbb{R}^{N\times N}$ such that:
\begin{align*}
    \mathbb{E}_{\{\epsilon_i\} \sim DPP(\mathbf{K})}\left[  \hat{F}(\theta)_U^{DPP}  \right] &= \mathbb{E}_{\{ \epsilon_i \}  \sim \{Ber(p_i) \} }\left[  \hat{F}(\theta)_U^{iid}  \right] \\
    &= \frac{1}{N} \sum_{i=1}^N h_\theta(\mathbf{v}^i)
\end{align*}
And:
\begin{equation*}
   \mathrm{Var}( \hat{F}(\theta)^{DPP}_U  ) <  \mathrm{
    Var}( \hat{F}(\theta)^{iid}_U  ) 
\end{equation*}
\end{theorem}

We show the corresponding result for the case when $\mathbb{F} = \nabla f_\sigma(\theta)$. The proof is exactly the same as in the case when considering any other type of function $F : \mathbb{R}^d \rightarrow \mathbb{R}^m$ defined as in Section \ref{base_eq}.

Let $\mathbf{K}$ be a marginal kernel matrix defining a DPP whose samples we index as $(\epsilon_1, \cdots, \epsilon_N)$ with $\epsilon_i \in \{ 0,1\}$ and such that the ensemble follows the DPP process. We consider the following subsampled ES estimator:

\begin{equation*}
    \hat{\nabla}^{DPP}_{U}  f_\sigma(\theta) = \frac{1}{N\sigma} \sum_{i\in S}    \frac{\epsilon_i }{  p_i}f(\theta + \sigma \mathbf{v}^i)\mathbf{v}^i
\end{equation*}

\begin{theorem}\label{theorem::unbiased_dpp_variance_reduction_evolution_strategies}
There exists a marginal kernel $\mathbf{K} \in \mathbb{R}^{N \times N}$ such that $\widehat{\mathrm{MSE}}(\hat{\nabla}^{\mathrm{DPP}}_{\mathrm{U}} f_\sigma(\theta)) < \widehat{\mathrm{MSE}}( \hat{\nabla}_{\mathrm{U}} f_\sigma(\theta)   ) $
\end{theorem}

\begin{proof}
Since $\mathbb{E}\left[  \hat{\nabla}^{\mathrm{DPP}}_U f_\sigma(\theta)  \right] = \mathbb{E}\left[ \hat{\nabla}_\mathrm{U} f_\sigma(\theta)      \right] $, it is enough to show the desired statement for the square norms of these vectors.

\begin{align*}
    \| \hat{\nabla}^{\mathrm{DPP}}_\mathrm{U} f_\sigma(\theta)  \|^2   &= \sum_{j=1}^d \left( \frac{1}{\sigma N} \sum_{i\in S}  \frac{\epsilon_i}{p_i} f(\theta + \sigma \mathbf{v}^i) \mathbf{v}^i(j)  \right)^2  \\
    &= \frac{1}{\sigma^2 N^2} \sum_{j=1}^d  \left( \sum_{i\in S} \frac{\epsilon_i}{p_i}f(\theta + \sigma \mathbf{v}^i) \mathbf{v}^i(j)  \right)^2\\
    &= \frac{1}{\sigma^2N^2} \sum_{j=1}^d\left( \sum_{i ,k \in S} \frac{\epsilon_i \epsilon_k}{p_ip_k} f(\theta + \sigma \mathbf{v}^i)f(\theta + \sigma \mathbf{v}^k) \mathbf{v}^i(j)\mathbf{v}^k(j)      \right)
    \end{align*}

Therefore:

\begin{align*}
    \mathbb{E}\left[ \| \hat{\nabla}^{\mathrm{DPP}}_{\mathrm{U}} f_\sigma(\theta)  \|^2  \right] &=   \mathbb{E}\left[  \frac{1}{\sigma^2N^2} \sum_{j=1}^d\left( \sum_{i ,k \in \mathcal{S}} \frac{\epsilon_i \epsilon_k}{p_ip_k} f(\theta + \sigma \mathbf{v}^i)f(\theta + \sigma \mathbf{v}^k) \mathbf{v}^i(j)\mathbf{v}^k(j)      \right)  \right] \\
    &= \frac{1}{\sigma^2 N^2} \sum_{j=1}^d \left(    \sum_{i,k } \frac{\mathbb{E}[ \epsilon_i \epsilon_k ]}{p_i p_k}    f(\theta + \sigma \mathbf{v}^i)f(\theta + \sigma \mathbf{v}^k) \mathbf{v}^i(j)\mathbf{v}^k(j)   \right) \\
     &= \frac{1}{\sigma^2 N^2} \sum_{j=1}^d \left(    \sum_{i\neq k} \frac{\mathbf{K}_{i,i} \mathbf{K}_{k,k} - \mathbf{K}_{i,k}^2}{p_i p_k}    f(\theta + \sigma \mathbf{v}^i)f(\theta + \sigma \mathbf{v}^k) \mathbf{v}^i(j)\mathbf{v}^k(j)   \right) +\\
     &\frac{1}{\sigma^2 N^2} \sum_{j=1}^d \left(    \sum_{i=1}^N \frac{\mathbf{K}_{i,i}}{p_i^2}    f^2(\theta + \sigma \mathbf{v}^i) \left(\mathbf{v}^i\right)^2(j)  \right) 
\end{align*}

Let $K_{i,i} = p_i$ for all $i$. The expression above becomes:

\begin{align*}
    \mathbb{E}\left[ \| \hat{\nabla}^{\mathrm{DPP}}_{\mathrm{U}} f_\sigma(\theta)  \|^2  \right]  &= \mathbb{E}\left[  \| \hat{\nabla}_{\mathrm{U}} f_\sigma(\theta)   \|^2 \right] - \frac{1}{\sigma^2 N^2} \sum_{j=1}^d \left( \sum_{i \neq k}    \frac{\mathbf{K}_{i,k}^2}{p_i p_k}f(\theta + \sigma \mathbf{v}_i) f(\theta + \sigma \mathbf{v}_k) \mathbf{v}_i(j) \mathbf{v}_k(j)    \right) \\
    &= \mathbb{E}\left[  \| \hat{\nabla}_{\mathrm{U}} f_\sigma(\theta)   \|^2 \right] - \frac{1}{\sigma^2 N^2} \sum_{i \neq k} \frac{\mathbf{K}_{i,k}^2}{p_i p_k} f(\theta + \sigma \mathbf{v}^i) f(\theta + \sigma \mathbf{v}^j) \left(\sum_j \mathbf{v}^i(j)\mathbf{v}^k(j) \right) \\
        &= \mathbb{E}\left[  \| \hat{\nabla}_{\mathrm{U}} f_\sigma(\theta)   \|^2 \right] - \underbrace{ \frac{1}{\sigma^2 N^2} \sum_{i \neq k} \frac{\mathbf{K}_{i,k}^2}{p_ip_k} f(\theta + \sigma \mathbf{v}^i) f(\theta + \sigma \mathbf{v}^j) \langle \mathbf{v}^i, \mathbf{v}^k \rangle }_{I}
\end{align*}

Let $\mathbf{V} \in \mathbb{R}^{d \times N}$ where the $i-$th column of $\mathbf{V}$ equals $\mathbf{v}^i$, and let $\mathbf{D} \in \mathbb{R}^{N\times N}$ a diagonal matrix such that $\mathbf{D}_{i,i} = \frac{ f(\theta + \sigma \mathbf{v}^i) }{p_i \sigma N} $. Let $\mathbf{K}^0 \in \mathbb{R}^{N\times N}$ be a matrix having zero diagonal entries and such that $\mathbf{K}^0_{i, j}  = \mathbf{K}_{i,j}$ with $i\neq j$. Similarly to the proof of Lemma \ref{lemma::variance_reduction_lemma_2}, let's focus on term I. 

\begin{align*}
    \frac{1}{\sigma^2 N^2} \sum_{i \neq k}\frac{ \mathbf{K}_{i,k}^2 }{p_i p_k} f(\theta + \sigma \mathbf{v}^i) f(\theta + \sigma \mathbf{v}^j) \langle \mathbf{v}^i, \mathbf{v}^k \rangle &= \sum_{i \neq k} \frac{\mathbf{K}_{i,k}^2}{\sigma^2 N^2 p_ip_k} f(\theta + \sigma \mathbf{v}^i) f(\theta + \sigma \mathbf{v}^j) \langle \mathbf{v}^i, \mathbf{v}^k \rangle \\
    &= \langle  \left(\mathbf{K}^0 \right)^2, \mathbf{D}^\top \mathbf{V}^\top \mathbf{D} \mathbf{V} \rangle
\end{align*}

We denote by $\left(\mathbf{K}^0\right)^2$ be the matrix $\mathbf{K}^0$ with entries squared. Where $ \langle \left(\mathbf{K}^0\right)^2, \mathbf{D}^\top \mathbf{V}^\top \mathbf{D} \mathbf{V} \rangle = \text{trace}( \left( \mathbf{K}^0\right)^2 \mathbf{D}^\top \mathbf{V}^\top \mathbf{D} \mathbf{V} )$. Define $\mathbf{K}^0$ in this way, for $i \neq j$. Let $\epsilon > 0$:

\begin{equation*}
    (\mathbf{K}^0)_{i,j} = \begin{cases}
            \epsilon &\text{if } f(\theta + \sigma \mathbf{v}^i) f(\theta + \sigma \mathbf{v}^j) \langle \mathbf{v}^i, \mathbf{v}^k \rangle  > 0 \\
            0 &\text{o.w.} 
        \end{cases}
\end{equation*}

Let $\mathbf{V}\mathbf{D}$ be the matrix with columns equal to $\mathbf{v}^i f(\theta + \sigma \mathbf{v}^i)$ and define $\mathbf{W} = \mathbf{V}\mathbf{D}$. Consider $\mathbf{J} = \mathbf{W}^\top \mathbf{W}$ and define $\mathbf{J}^0$ be the matrix $\mathbf{J}$ without its diagonal entries. Since $N \geq d+2$, Lemma \ref{lemma::vectors_negative_correlations} there must be at least two positive non diagonal entries of $J$ and therefore in this case $\langle  \left(\mathbf{K}^0 \right)^2, \mathbf{D}^\top \mathbf{V}^\top \mathbf{D} \mathbf{V} \rangle > 0$. 

If $\mathbf{K}_{i,i} = p_i < 1$ for all $i$ then following an argument similar to the proof of \ref{lemma::variance_reduction_lemma_2}, we conclude there exists $\epsilon > 0$ such that $0 \prec \mathbf{K} \prec \mathbb{I}_d$ such that $\widehat{\mathrm{MSE}}(\hat{\nabla}^{\mathrm{DPP}}_{\mathrm{U}} f_\sigma(\theta)) < \widehat{\mathrm{MSE}}( \hat{\nabla}_{\mathrm{U}} f_\sigma(\theta)   ) $ as desired.

\end{proof}

Theorem \ref{theorem::auxiliary_dpp_variance_1}, yields the following corollary (corresponding to Corollary \ref{corollary::unbiased_ES_variance_reduction} in the main text). Under i.i.d. uniform sampling ($p_i = p$ for all $i$):

\begin{corollary}
Let $\mathbf{v}^1, \cdots, \mathbf{v}^N \sim \mathcal{N}(0, I_d)$ be normally distributed sensings sampled i.i.d. Let $\hat{\nabla}_U f_\sigma(\theta)$ and $\hat{ \nabla }_U^{DPP} f_\sigma(\theta)$ be subsampled gradients with $p_i = p < 1$ for all $i$ where $\hat{\nabla}_U^{DPP} f_\sigma(\theta)$ is produced with a kernel as in Theorem \ref{theorem::unbiased_dpp_variance_reduction}. The following hold:
\begin{equation*}
    \mathbb{E}\left[ \hat{\nabla}_U^{DPP} f_\sigma(\theta)   \right] = \mathbb{E}\left[ \hat{\nabla}_U f_\sigma(\theta)  \right] = \nabla f_\sigma(\theta) 
\end{equation*}
And:
\begin{equation*}
    \widehat{\mathrm{MSE}}(\hat{\nabla}^{DPP}_{U} f_\sigma(\theta)) < \widehat{\mathrm{MSE}}( \hat{\nabla}_{U} f_\sigma(\theta)   )
\end{equation*}

\end{corollary}

This corollary implies that picking the right Kernel, subsampling perturbations from a DPP process when these perturbations are all i.i.d. Gaussian vectors, yields an unbiased estimator of the smoothed gradient $\nabla f_\sigma(\theta)$ with less variance (in this case equal to the mean squared error) than a naive subsampled gradient estimator that subsamples the $\{ \mathbf{v}^i \}$ perturbations each with probability $p$.

\textbf{Biased subsampling}

The goal of this section is to show that for any i.i.d. subsampling strategy to build a biased estimator for the ES gradient, there exists a DPP kernel such that the DPP unbiased subsampling estimator achieves less mean squared error (MSE) than the i.i.d. one. 

Define the biased downsampled estimator as:
\begin{equation}
    \hat{F}(\theta)_B = \frac{1}{N} \sum_{i = 1}^N \frac{\epsilon_i}{w_i}  h_\theta(\mathbf{v}^i).
\end{equation}

Theorem \ref{theorem::unbiased_dpp_variance_reduction} 
and Corollary \ref{corollary::unbiased_ES_variance_reduction} of the main text can be generalized to the case of biased estimators. 
Borrowing notation from the previous section, and assuming access to an ensemble $\{ w_i \}$ of importance weights, we get as a biased equivalent version of Theorem \ref{theorem::unbiased_dpp_variance_reduction}:



\begin{theorem}\label{theorem::biased_dpp_variance_reduction}
If $N \geq d+2$ and $p_i < 1$ there exists a Marginal Kernel $\mathbf{K} \in \mathbb{R}^{N\times N}$ such that:
\begin{equation*}
    \mathbb{E}_{\{ \epsilon_i \} \sim DPP(\mathbf{K})}\left[ \hat{F}(\theta)_B^{\mathrm{DPP}}   \right] =    
    \mathbb{E}_{\{ \epsilon_i \} \sim\{ Ber(p_i) \} }\left[ \hat{F}(\theta)^{\mathrm{iid}}_B   \right],
\end{equation*}
and furthermore $\mathrm{MSE}( \hat{F}(\theta)^{\mathrm{DPP}}_B  ) \leq    \mathrm{MSE}( \hat{F}(\theta)^{\mathrm{iid}}_B )$, 
where the comparison mean equals $\mu = \hat{F}(\theta) = \frac{1}{N} \sum_{i=1}^N  h_\theta( \mathbf{v}^i)$.
\end{theorem}
\begin{proof}
The following equalities hold:
\begin{align*}
     \mathrm{MSE}( \hat{F}(\theta)^{\mathrm{DPP}}_B  ) &=  \mathrm{Var}( \hat{F}(\theta)^{\mathrm{DPP}}_B  ) +\\
     &\quad\left\| \mathbb{E}\left[  \hat{F}(\theta)^{\mathrm{DPP}}_B- \hat{F}(\theta) \right]\right\|^2 \\
        \mathrm{MSE}( \hat{F}(\theta)^{\mathrm{iid}}_B  ) &=  \mathrm{Var}( \hat{F}(\theta)^{\mathrm{iid}}_B  ) +\\
     &\quad\left\| \mathbb{E}\left[  \hat{F}(\theta)^{\mathrm{iid}}_B- \hat{F}(\theta) \right]\right\|^2 
\end{align*}

Since the expectations of $\hat{F}(\theta)^{\mathrm{iid}}_B$ and $\hat{F}(\theta)_B^{\mathrm{DPP}}$ agree, and as a consequence of 
Theorem \ref{theorem::unbiased_dpp_variance_reduction}, we can produce a kernel $\mathbf{K}$ such that:
\begin{equation*}
  \mathrm{Var}( \hat{F}(\theta)^{\mathrm{DPP}}_B  ) <  \mathrm{Var}( \hat{F}(\theta)^{\mathrm{iid}}_B), 
\end{equation*}
the result follows.
\end{proof}

As a consequence of Theorem \ref{theorem::biased_dpp_variance_reduction}, the biased downsampled versions $\hat{\nabla}_B^{\mathrm{iid}} f_\sigma(\theta)$ 
and $\hat{\nabla}_B^{\mathrm{DPP}} f_\sigma(\theta)$ of the ES gradient estimator $\nabla f_\sigma(\theta)$ 
satisfy an analogous version of Corollary \ref{corollary::unbiased_ES_variance_reduction} where $\mathrm{Var}$ is substituted by $\mathrm{MSE}$.

The proofs of Theorem \ref{theorem::auxiliary_dpp_variance_1} and \ref{theorem::biased_dpp_variance_reduction} can be used 
to produce an algorithm to find kernel matrix $\mathbf{K}$ reducing MSE. The results of the previous section can be extended to the case of biased sampling estimators. These result from the case when the importance weights are different from $p_i$. 

Similarly to the previous section, the following theorem holds. Defining $\hat{\nabla}_{\mathrm{B}} f_\sigma(\theta)$ and $\hat{\nabla}^{\mathrm{DPP}}_{\mathrm{B}} f_\sigma(\theta)$ as:

\begin{enumerate}
    \item $\hat{\nabla}_{\mathrm{B}} f_\sigma(\theta) = \frac{1}{\sigma N} \sum_{i=1}^N  \frac{\epsilon_i}{w_i} \mathbf{v}^i f(\theta + \sigma \mathbf{v}^i)$ where $\{ w_i \}_{i=1}^N$ is a set of importance weights and $\epsilon_i \sim \mathrm{Ber}(p_i)$ for some probabilities ensemble $\{ p_i \}$\\
    \item $\hat{\nabla}_{\mathrm{B}}^{\mathrm{DPP}} f_\sigma(\theta) = \frac{1}{N\sigma} \sum_{i \in \mathcal{S}} \frac{\epsilon}{w_i} f(\theta + \sigma \mathbf{v}^i) \mathbf{v}^i$.
\end{enumerate}

In this case, the corresponding version of Theorem \ref{theorem::auxiliary_dpp_variance_1} is:
 
\begin{theorem}
There exists a marginal kernel $\mathbf{K} \in \mathbb{R}^{N \times N}$ such that $\widehat{\mathrm{MSE}}(\hat{\nabla}^{\mathrm{DPP}}_{\mathrm{B}} f_\sigma(\theta)) < \widehat{\mathrm{MSE}}( \hat{\nabla}_{\mathrm{B}} f_\sigma(\theta)   ) $.
\end{theorem}

\begin{proof}
The mean squared errors $\widehat{\mathrm{MSE}}(\hat{\nabla}^{\mathrm{DPP}}_{\mathrm{B}} f_\sigma(\theta))$ and $\widehat{\mathrm{MSE}}( \hat{\nabla}_{\mathrm{B}} f_\sigma(\theta) ) $ can be written as:
\begin{align*}
    \widehat{\mathrm{MSE}}(\hat{\nabla}^{\mathrm{DPP}}_{\mathrm{B}} f_\sigma(\theta)) &= \mathrm{Var}( \hat{\nabla}^{\mathrm{DPP}}_{\mathrm{B}} f_\sigma(\theta)) + \underbrace{ \left\| \mathbb{E}\left[ \hat{\nabla}^{\mathrm{DPP}}_{\mathrm{B}} f_\sigma(\theta))  \right] -  \nabla f_\sigma(\theta) \right\|^2}_{I} \\
    \widehat{\mathrm{MSE}}(\hat{\nabla}_{\mathrm{B}} f_\sigma(\theta)) &= \mathrm{Var}( \hat{\nabla}_{\mathrm{B}} f_\sigma(\theta)) + \underbrace{ \left\| \mathbb{E}\left[ \hat{\nabla}_{\mathrm{\mathrm{B}}} f_\sigma(\theta))  \right] -  \nabla f_\sigma(\theta) \right\|^2}_{II} 
\end{align*}

The bias terms $I$ and $II$ are always equal since $\mathbb{E}\left[ \hat{\nabla}_{\mathrm{B}} f_\sigma(\theta))  \right] = \mathbb{E}\left[ \hat{\nabla}^{\mathrm{DPP}}_{\mathrm{B}} f_\sigma(\theta))  \right]$.

The remainder of the proof is exactly the same as in Theorem \ref{theorem::unbiased_dpp_variance_reduction}. 

\end{proof}

\subsection{DPP Connections with orthogonality}

In this section we flesh out some connections between structured sampling via DPPs and structured sampling via orthogonal directions such as in \cite{montreal}. We show that in some way DPP structured sampling subsumes orthogonal sampling. We start showing Lemma \ref{lemma::covariance_kernel_equivalence}, leading to Theorem \ref{theorem::orthogonal_connection1}, (Theorem \ref{theorem::orthogonal_connections} in the main text).

In what follows assume $\mathcal{X} = \{ \mathbf{x}^1, \cdots, \mathbf{x}^N\}$ with $\mathbf{x}^i \in \mathbb{R}^d$ and let $\phi : \mathbb{R}^d  \rightarrow \mathbb{R}^D$ be a possibly infinite feature map $\phi$.

\begin{lemma}\label{lemma::covariance_kernel_equivalence}
Let $\mathbf{W} \in \mathbb{R}^{N\times N}$ such that $\mathbf{W}_{i,j} = \langle \phi(\mathbf{x}^i) , \phi(\mathbf{x}^j) \rangle$ for some a $D-$dimensional feature map $\phi$. Let $A \subseteq [N]$. The nonzero eigenvalues of the principal minor $W_A$ equal the nonzero eigenvalues of $\sum_{i \in A} \phi(\mathbf{x}^i)\phi^\top(\mathbf{x}^i)$.
\end{lemma}

\begin{proof}
Let $A = \{  i_1, \cdots, i_{|A|} \}$ and define $B_A = \left[ \phi(\mathbf{x}^{i_1}) \cdots \phi(\mathbf{x}^{|A|})   \right] \in \mathbb{R}^{D \times |A|}$. It follows immediately that:

\begin{equation*}
    \mathbf{W}_A = \mathbf{B}_A^\top \mathbf{B}_A
\end{equation*}

Assume the SVD decomposition of $\mathbf{B}_A = \mathbf{U}_A^\top \mathbf{D}_A \mathbf{V}_A$ with $\mathbf{U}_A \in \mathbb{R}^{ D\times D} $, $\mathbf{D}_A \in \mathbb{R}^{D \times |A|} $, and $\mathbf{V}_A \in \mathbb{R}^{|A| \times |A|}$. And thus:

\begin{equation*}
    \mathbf{W}_A = \mathbf{V}_A^\top \underbrace{\mathbf{D}_A \mathbf{D}_A^\top}_{I} \mathbf{V}_A
\end{equation*}

Observe that:

\begin{equation*}
    \sum_{i \in A} \phi(\mathbf{x}^i) \phi^\top(\mathbf{x}^i) = \mathbf{B}_A \mathbf{B}_A^\top 
\end{equation*}

And substituting the SVD decomposition of $B_A$ yields:

\begin{equation*}
    \sum_{i\in A} \phi(\mathbf{x}^i) \phi^\top (\mathbf{x}^i) = \mathbf{U}_A^\top \underbrace{\mathbf{D}_A^\top \mathbf{D}_A}_{II} \mathbf{U}_A
\end{equation*}
Since the nonzero entries of $I$ and $II$ are the same, we conclude the nonzero eigenvalues of $\mathbf{W}_A$ and  of 
$\sum_{i \in A} \phi(\mathbf{x}^i ) \phi^\top(\mathbf{x}^i)$ coincide.
\end{proof}

We now show a relationship between orthogonality and DPPs.

\begin{theorem}\label{theorem::orthogonal_connection1}
Let $\mathbf{L} \in \mathbb{R}^{N \times N} $ be an $\mathbf{L}-$ensemble such that $\mathbf{L}_{i,j} = \langle \phi(\mathbf{x}^i), \phi(\mathbf{x}^j) \rangle$, where $\| \Phi(\mathbf{x}^i)\| = 1$ for all $i \in [N]$. Let $k \in \mathbb{N}$ with $k \leq N$ and assume there exist $k$ samples $\mathbf{x}^{i_1}, \cdots, \mathbf{x}^{i_k}$ in $\mathcal{X}$ satisfying $\langle \phi(\mathbf{x}^{i_j}) , \phi(\mathbf{x}^{i_l})\rangle = 0$ for all $j, l \in [k]$. If $\mathbb{P}_k$ denotes the DPP measure over subsets of size $k$ of $[N]$ defined by $\mathbf{L}$, the most likely outcomes from $\mathbb{P}_k$ are the size $k$ pairwise orthogonal subsets of $\mathcal{X}$.
\end{theorem}


\begin{proof}
Recall that $\mathbb{P}_k\propto \det( \mathbf{L}_A)$. Observe also that, since all eigenvalues of $\mathbf{L}_A$ are nonnegative, if we assume the determinant of $\mathbf{L}_A$ to be nonnegative, by the arithmetic-geometric inequality:

\begin{equation}\label{equation::inequality_det_trace}
    \left( \det( \mathbf{L}_A )\right)^{1/k} \leq \frac{\mathrm{tr}( \mathbf{L}_A)} {k} = 1
\end{equation}

Since the determinant equals the product of the eigenvalues while the trace is the sum. Equality holds iff all of the eigenvalues are equal to $1$. Let $A$ be a subset of size $k$ such that all points are pairwise orthogonal after the map $\Phi$, then $\det(\mathbf{L}_A = 1$. Furthermore, if $\det( \mathbf{L}_A  ) = 1$, then the set of points $\{ \phi(\mathbf{x}^i) \}_{i \in A}$ must be orthogonal.

As a consequence of inequality \ref{equation::inequality_det_trace}, the equality $\det( \mathbf{L}_A ) = t^k$ can only hold if all eigenvalues of $\mathbf{L}_A$ equal $1$. We show this implies all the vectors must be orthogonal. 

Let $A = \{ i_1, \cdots, i_{|A|}\}$ and write $L_A^{(t)} = \left( \mathbf{B}_A \right)^\top \mathbf{B}_A $ where $B_{A} = \left[  \phi(\mathbf{x}^{i_1}) \cdots \phi(\mathbf{x}^{i_{|A|}})\right]$. As a consequence of Lemma \ref{lemma::covariance_kernel_equivalence}, the nonzero eigenvalues of $L_A^{(t)}$agree with the nonzero eigenvalues of $\Sigma = \sum_{i \in A}  \phi(\mathbf{x}^i)\phi^\top(\mathbf{x}^i)$.

Since by assumption $\| \Sigma\| = t$, and $\| \phi(x_i) \| =1$ for all $i$:

\begin{equation*}
    \phi^\top(\mathbf{x}^i) \Sigma \phi(\mathbf{x}^i) \leq 1
\end{equation*}

Expanding this equation by substituting the value of $\Sigma$, we get:
$
\phi^\top(\mathbf{x}^i)\Sigma \phi(\mathbf{x}^i) = \sum_{j \in A}  \langle \Phi(x_j), \Phi(x_i) \rangle^2 \leq 1
$

Since the term corresponding to $j = i$ already equals $1$, the remaining terms must be zero. This finishes the proof.

This result implies that the subsets of points of size $k$ with the largest mass are those corresponding to pairwise orthogonal ensembles. This finishes the proof.

\end{proof}

\section{Experiment Details}

\subsection{Code}
\label{code}

Here we include some simple code to implement DPPMC using python 3.x.

\begin{python}

import numpy as np
from pydpp.dpp import DPP

d = 10 # this will be the dimensionality of your problem
rho = 5 # this is a hyper-parameter
cov = np.eye(d) # this will be your nonisotropic covariance matrix
mu = np.repeat(0, d)
A = np.random.multivariate_normal(mu, cov, d * rho)

dpp = DPP(A)
dpp.compute_kernel(kernel_type = 'rbf')
idx = dpp.sample_k(d) # returning to original dimensionality, optional
A = A[idx]

# we now evaluate these samples.
\end{python}

This code is simple to include in any setting where samples are drawn from a nonisotropic distribution.

\subsection{Optimal Choice of $\rho$}
\label{ro}

Here we demonstrate the impact of $\rho$ by performing an ablation study using the $\mathrm{CMA}$-$\mathrm{ES}$ experiments. In order to measure the importance of this parameter, we test the following values: $\rho = 2,5,10,20$, and measure the mean performance across three seeds after $100$ function evaluations.

As we can see in Figure \ref{ablationrho}, in most cases an increase in $\rho$ leads to a monotonic improvement in performance. This however comes at an increase in computational cost, and as such it is important to consider the trade-off between the cost of evaluating the function vs. the $\mathrm{DPPMC}$ algorithm when choosing an optimal $\rho$ for a given problem. In our experiments we choose $\rho=10$ since this value is sufficient to achieve meaningful performance gains, demonstrating the effectiveness of our approach.

\subsection{Reinforcement Learning Experiments}
We provide details on the reinforcement learning experiments as follows.
\paragraph{Benchmark Environments.} Reinforcement learning tasks are identified by a state space $\mathcal{S}$ and an action space $\mathcal{A}$. The benchmark environments consist of HalfCheetah-v2 ($|\mathcal{S}| = 17, |\mathcal{A}|=6$), Swimmer-v2 ($|\mathcal{S}| = 8, |\mathcal{A}|=2$), Reacher-v2 ($|\mathcal{S}| = 11, |\mathcal{A}|=2$) and Walker2d-v2 ($|\mathcal{S}| = 17, |\mathcal{A}|=6$). Each task takes the sensory inputs of the robot as states $s_t \in \mathcal{S}$ and motor/position controls as actions $a_t \in \mathcal{A}$. All environments are simulated via OpenAI gym \citep{brockman2016openai}.

\paragraph{Policy Architecture.} We encode the policy $\pi_\theta: \mathcal{S} \mapsto  \mathcal{A}$ with feed-forward network parameter $\theta$. The architecture varies across tasks: for Swimmer-v2 and Reacher-v2, we have two hidden layers each with $16$ units; for HalfCheetah-v2 and Walker2d-v2, we have two hidden layers each with $32$ units. Each hidden layer is combined with a $\text{tanh}$ non-linear function activation. The output layer does not have non-linear function activation. For each hidden layer, instead of a fully-connected structure, we adopt a low displacement rank neural network \citep{choromanski_stockholm} for a compact representation.

\paragraph{Implementations and Common Hyper-parameters.} All ES algorithms are implemented with Numpy \citep{van2011numpy}. To make our implementations parallelizable, we have made heavy reference to the Ray open source project \citep{moritz2018ray}. At each iteration, the ES algorithms (including Guided ES, Trust Region ES and CMA-ES) all require sampling $m$ perturbation directions for function evaluations. We set $m$ to be the dimension $d$ of the policy parameter $\theta$. Gradient based optimizations are all carried out using Adam Optimizer \citep{kingma2014adam} with best learning rates chosen from $\alpha \in \{0.5,0.1,0.05,0.01\}$. 

\paragraph{DPPMC Hyper-parameters.} We use a fixed RBF-kernel for all experiments: recall that a RBF-kernel takes the form $K(\mathbf{x},\mathbf{y}) = \exp(\frac{-|\mathbf{x} - \mathbf{y}|^2}{2\sigma^2})$, we set $\sigma = 0.5$. The kernel parameter $\sigma$ is manually set such that the DPPMC variants achieve good performance while the computations remain numerically stable.

\paragraph{Hyper-parameters for Guided ES.} We follow the recipe of Guided ES \citep{metz} to set up hyper-parameters. The DPPMC variant requires constructing a sample pool of size $\rho m$, we choose $\rho = 10$ for our experiments. The Guided ES achieves performance gains over vanilla ES by constructing non-isotropic distribution for gradient sensing, which allows for exploring subspaces where the true gradients lie. We further improve upon Guided ES with significant gains in sample efficiency.

\paragraph{Hyper-parameters for Trust Region ES.} We follow the recipe of Trust Region ES \citep{rbo} to set up hyper-parameters. Trust Region ES has two variants: (1) using ridge regression to compute update directions (Ridge); (2) using Monte-Carlo samples to estimate update directions (MC). Both variants require re-using $\delta m$ samples and function evaluations from the previous iteration, here we set $\delta = 0.2$ so that the algorithm achieves $\approx 20\%$ sample gains. On top of Trust Region ES, the DPPMC variant further improves sample efficiency as demonstrated in the main paper. We refer readers to \citep{rbo} for a detailed description of the algorithm.

\begin{figure}[H]
\begin{minipage}{0.99\textwidth}
	\centering \subfigure[$\mathrm{Cigar}$]{\includegraphics[keepaspectratio, width=0.35\textwidth]{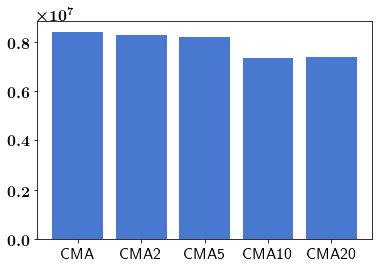}}
	\centering \subfigure[$\mathrm{Rastrigin}$]{\includegraphics[keepaspectratio, width=0.35\textwidth]{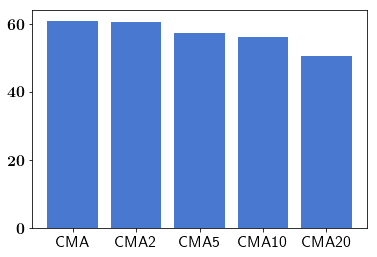}}  
\end{minipage}
\begin{minipage}{0.99\textwidth}
	\centering \subfigure[$\mathrm{Rosenbrock}$]{\includegraphics[keepaspectratio, width=0.35\textwidth]{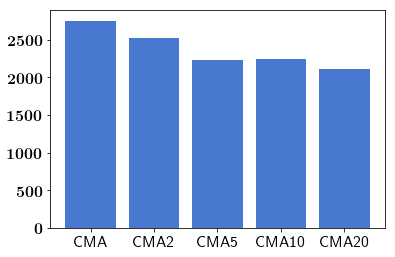}}
	\centering \subfigure[$\mathrm{Sphere}$]{\includegraphics[keepaspectratio, width=0.35\textwidth]{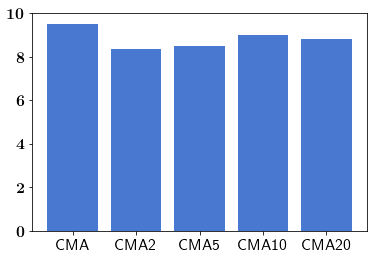}}  
\end{minipage}
	\caption{Comparison of $\mathrm{CMA}$-$\mathrm{ES}$ without $\mathrm{DPPMC}$ vs. with $\mathrm{DPPMC}$ for $\rho = 2,5,10,20$.}
	\label{fig:cmaes_exps} 
\label{ablationrho}
\end{figure}


\end{document}